\documentclass{article}
\usepackage{microtype}
\usepackage{graphicx}
\usepackage{subfigure}
\usepackage{booktabs} % for professional tables
\usepackage{hyperref}
\usepackage{times}
\usepackage{helvet}
\usepackage{courier}
\usepackage{subfigure} 
\usepackage{caption}
\usepackage{dblfloatfix}
\usepackage{fixltx2e}
\usepackage{color}
\usepackage[accepted]{icml2018}
\usepackage{color}
\usepackage{mathtools}
\usepackage{amsthm}
\usepackage{amsmath}
\usepackage{amsfonts}
\usepackage{eqparbox,array}
\newtheoremstyle{break}
  {\topsep}{\topsep}%
  {\itshape}{}%
  {\bfseries}{}%
  {\newline}{}%
\theoremstyle{break}
\newtheorem{definition}{Definition}
\usepackage[utf8]{inputenc}
\usepackage[english]{babel}
\usepackage{amsthm}
\newtheorem{theorem}{Theorem}
\newtheorem*{theorem*}{Theorem}

\DeclarePairedDelimiter{\ceil}{\lceil}{\rceil}
\newtheorem{corollary}{Corollary}[theorem]
\newtheorem{lemma}{Lemma}
\renewcommand\algorithmiccomment[1]{%
  \hfill\#\ \eqparbox{COMMENT}{#1}%
}
\icmltitlerunning{Approximating Optimal Discounted TSP Using Local Policies}

\begin{document}

\twocolumn[
\icmltitle{Hierarchical Reinforcement Learning:\\
Approximating Optimal Discounted TSP Using Local Policies}

\begin{icmlauthorlist}
\icmlauthor{Tom Zahavy}{goo,tech}
\icmlauthor{Avinatan Hasidim}{goo,bi}
\icmlauthor{Haim Kaplan}{goo,tau}
\icmlauthor{Yishay Mansour}{goo,tau}
\end{icmlauthorlist}

\icmlaffiliation{goo}{Google}
\icmlaffiliation{bi}{Bar Ilan Univeristy}
\icmlaffiliation{tau}{Tel Aviv University}
\icmlaffiliation{tech}{The Technion}

\icmlcorrespondingauthor{Tom Zahavy}{tomzahavy@gmail.com}

\icmlkeywords{Reinforcement Learning}

\vskip 0.3in
]


% This command creates the footnote in the first column

\printAffiliationsAndNotice{}  % leave blank if no need to mention equal contribution

\begin{abstract}
In this work, we provide theoretical guarantees for reward decomposition in deterministic MDPs. Reward decomposition is a special case of Hierarchical Reinforcement Learning, that allows one to learn many policies in parallel and  combine them into a composite solution. Our approach builds on mapping this problem into a Reward Discounted Traveling Salesman Problem, and then deriving approximate solutions for it. In particular, we focus on approximate solutions that are local, i.e., solutions that only observe information about the current state. Local policies are easy to implement and do not require substantial computational resources as they do not perform planning. While local deterministic policies, like Nearest Neighbor, are being used in practice for  hierarchical reinforcement learning, we propose three stochastic policies that guarantee better performance than any deterministic policy. 
\end{abstract}

\section{Introduction}
%hierarchy
One of the unique characteristics of human problem solving is the ability to represent the world on different granularities. When we plan a trip, we first choose the destinations we want to visit and only then decide what to do at each destination. Hierarchical reasoning enables us to map the complexities of the world around us into simple plans that are computationally tractable to reason. Nevertheless, the most successful Reinforcement Learning (RL) algorithms are still performing planning with only one abstraction level. 

%DRL->HRL
RL provides a general framework for optimizing decisions in dynamic environments. However, scaling it to real-world problems suffers from the curses of dimensionality; that is, coping with exponentially large state spaces, action spaces, and long horizons. One approach deals with large state spaces by introducing a function approximation to the value function or policy, making it possible to generalize across different states. Two famous examples are TD-Gammon \citep{tesauro1995temporal} and the Deep Q Network (DQN) \citep{mnih2015human}, both introduced a Deep Neural Network (DNN) to approximate the value function leading to a high performance in solving Backgammon and video games. A different approach deals with long horizons by using a policy network to search among game outcomes efficiently \citep{silver2016mastering}, leading to a super-human performance in playing Go, Chess, and Poker \citep{silver2016mastering,silver2017mastering,moravvcik2017deepstack}. However, utilizing this approach when it is not possible to simulate the environment by doing model-based RL is still an open problem \cite{oh2015action}.

% H-RL
A long-standing approach for dealing with long horizons is to introduce hierarchy into the problem (see \citet{barto2003recent} for a survey). We will focus on the options framework \citep{sutton1999between}, a two-level hierarchy formulation where options (local policies that map states to actions) are learned to achieve subgoals, while the policy over options selects among options to accomplish the final goal of a task. Recently, it was demonstrated that learning a selection rule among pre-defined options using a DNN delivers promising results in challenging environments like Minecraft and Atari \citep{tessler2017deep,kulkarni2016hierarchical,oh2017zero}; other studies have shown that it is possible to learn options jointly with a policy-over-options end-to-end \citep{vezhnevets2017feudal,bacon2017option}.

% reward decomposition
In this work, we focus on a specific type of hierarchy - reward function decomposition - that dates back to the works of \citep{humphrys1996action,karlsson1997learning} and has been studied among different research groups recently \citep{van2017hybrid}. In this formulation, each option $i$ learns to maximize a local reward function $R_i$, while the final goal is to maximize the sum of rewards $R_M=\sum R_i$. Each option is trained separately and provides a value function and an option $o _i $ for the policy over options, which then uses the local values to select among options. That way, each option is responsible for solving a simple task, and the options are learned in parallel across different machines. While the higher level policy can be trained using SMDP algorithms \cite{sutton1999between}, different research groups suggested using pre-defined rules to select among options. For example, choosing the option with maximal value function \cite{humphrys1996action,barreto2016successor}, or choosing the action that maximizes the sum of option value functions \cite{karlsson1997learning}. By using pre-defined rules, we can derive policies for MDP $M$ by learning options (and without learning in MDP $M$), such that learning is fully decentralized. Although in many cases one can reconstruct from the options the original MDP, doing it would defeat the entire purpose of using options. In this work we concentrate on local rules that select among available options.

Even more specifically, we consider a set of $n$ MDPs $\{M_i\}_{i=1}^n = \{S,A,P,\gamma,R_i\}_{i=1}^n$ with deterministic dynamics that share all components but the reward. Given a set of options, one per reward,  with an optimal policy for collecting the reward, we are interested in deriving an optimal policy for collecting all the rewards, i.e., solving MDP $M = \{S,A,P,R_M = \sum_{i=1}^k R_i \}$. In this setting, an optimal policy for $M$ can be derived by solving the SMDP $M_s = \{S,O,P,R_M = \sum_{i=1}^k R_i \}$, whose actions are the optimal policies for collecting single rewards. 

Specifically, we focus on collectible rewards, a special type of reward that is very common in 2D and 3D navigation domains like Minecraft \citep{tessler2017deep,oh2017zero}, DeepMindLab \citep{teh2017distral,beattie2016deepmind} and VizDoom \citep{kempka2016vizdoom}. The challenge when dealing with collectible rewards is that the state space changes each time we collect a reward (One can think of the subset of available rewards is part of the state). Since all the combinations of remaining items have to be considered, the state space grows exponentially with the number of rewards.

Here, we show that solving an SMDP under these considerations is equivalent to solving a Reward Discounted Traveling Salesman Problem (RD-TSP) \citep{blum2007approximation}. Similar to the classical Traveling Salesman Problem (TSP), computing an optimal solution to  RD-TSP is NP-hard, and furthermore, it is NP-hard to get an approximate solution of value at least 99.5\% the optimal discounted return \citep{blum2007approximation} in polynomial time.\footnote{That is, any algorithm for approximating the optimal return for RD-TSP to within a factor larger than $0.995$ must have a worst-case running time that grows faster than any polynomial (assuming the widely believed conjecture that P $\ne$ NP).} 

A brute force approach for solving the RD-TSP requires evaluating all the $n!$ possible tours connecting the $n$ rewards.
We can also adapt the Bellman–Held–Karp dynamic programming algorithm for TSP \citep{bellman1962dynamic,held1962dynamic} 
to solve RD-TSP (see Algorithm \ref{alg:held-karp} in the appendix). This scheme is 
identical to tabular Q-learning on SMDP $M_s$, and still requires exponential time.\footnote{The Hardness results for RD-TSP do not rule out efficient solutions for special MDPs. For example, we provide, in the appendix,
exact polynomial-time solutions for the case in which  the MDP is a line and when it is a star.}

\citet{blum2007approximation} proposed a polynomial time planning algorithm for RD-TSP that computes a policy
which collects at least $0.15$ fraction of the optimal discounted return, which was later improved to $0.19$
 \citep{farbstein2016discounted}. 
 These planning algorithms need to know the entire SMDP in order to compute their approximately optimal policies.
 
In contrast, in this work, we focus on deriving and analyzing  policies that use only \textbf{local} information\footnote{Observe only the value  of each option from the current state.} to make decisions; such \textbf{local policies} are simpler to implement, more efficient, and do not need to learn in $M$ nor $M_s$. 
The reinforcement learning community is already using simple local approximation algorithms for RD-TSP.
We hope that our research will provide important theoretical support for comparing local heuristics, and in addition introduce new reasonable local heuristics. 
Specifically,
we prove worst-case guarantees on the reward collected by these algorithms relative to the reward of the optimal RD-TSP tour. We also prove bounds on the maximum relative reward that such local algorithms can collect.
In our experiments, we compare the performance of these local algorithms.
In particular, our main contributions are as follows.

{\bf Our results:} We establish impossibility results for local policies, showing that no deterministic local policy can guarantee a reward larger  than $ 24 \text{OPT} / n $ for every  MDP, and no stochastic policy can guarantee a reward larger than $8 \text{OPT} / \sqrt{n}$  for every  MDP. These impossibility results imply that the Nearest Neighbor (NN) algorithm that iteratively collects the closest reward (and thereby a total of at least $OPT/n$ reward) is optimal up to constant factor amongst all deterministic local policies.

On the positive side, we propose three simple stochastic policies that outperform NN. The best of them combines NN with a Random Depth First Search (RDFS) and guarantees performance of at least $\Omega \left( \text{OPT} / \sqrt{n} \right)$ when OPT achieves $\Omega \left(n \right)$,  and at least $\Omega (\text{OPT} / n^{\frac{2}{3}} )$ in the general case.
Combining NN with jumping to a random reward and sorting the rewards by their distance from it, has a slightly worse guarantee.
A simple modification of the NN to first jump to a random reward and continues NN from there, already improves the guarantee to $O(OPT \log(n) /n)$.

\section{Problem formulation}
\label{sec:problem}
We now define our problem explicitly, starting from a general transfer framework in Definition \ref{def:1}, and 
then the more specific transfer learning setting of
of \emph{collectible reward decomposition} in Definition \ref{def:2}. 
\begin{definition}[General Transfer Framework] 
\label{def:1}
Given a set of MDPs $\{M_i\}_{i=1}^n =\{S,A,P,\gamma,R_i\}_{i=1}^n$ and an MDP $M=\{S,A,P,\gamma,R_M = f \left( R_1,...,R_n \right) \}$ that differ only by their reward signal, derive an optimal policy for $M$ given the optimal policies for $\{M_i\}_{i=1}^n$.
\end{definition}
Definition \ref{def:1} describes a general transfer learning problem in RL. Similar to \cite{barreto2016successor}, our transfer framework assumes a set of MDPs sharing all but the reward signal. We are interested in transfer learning, i.e., using quantities that were learned from the MDPs $\{M_i\}_{i=1}^n$ on the MDP $M$. More specifically for model-free RL, given a set of optimal options and their value functions $\{o^* _i,V^*_i\}_{i=1}^n$, we are interested in zero-shot transfer to MDP $M$, i.e., deriving policies for solving $M$ without learning in $M$. 
\begin{definition}[Collectible Reward Decomposition] 
\label{def:2}
1. \textbf{Reward Decomposition:} the reward in $M$ represents the sum of the local rewards: $R_M = \sum _{i=1} ^ n R_i$.\\
2. \textbf{Collectible Rewards:} each reward signal $\{R_i\}_{i=1}^n$ represents a collectible prize, i.e., $R_i(s,a) = 1$ iff $s = s_i, \enspace a=a_i$ for some particular state $s_i$ and action $a_i$  and $R_i(s,a) = 0$ otherwise. In addition, each reward can only be collected once.\\
3. \textbf{Deterministic Dynamics:} $P$ is a deterministic transition matrix, i.e., for each action $a$ each row of $P^a$ has exactly one value that equals $1$, and all its other values equal zero.
\end{definition}
Property 1 in Definition \ref{def:2} requires $R_M$ to be a decomposition of the previous rewards, and Property 2 requires each local reward to be a collectible prize. While limiting the generality, models that satisfy these properties has been investigated in theory and simulation \citep{oh2017zero,barreto2016successor,tessler2017deep,higgins2017darla,van2017hybrid,humphrys1996action}.
Now, given that the value functions of the local policies are optimal, the shortest path from a reward $i$ to a reward $j$ is given by following option $o_j$ from state $i$. In addition, the length of the shortest path from $i$ to $j$, denoted by $d_{i,j}$, is given by the value function, since $V_j(i) = \gamma ^ {d_{i,j}}$. Notice that in any state, an optimal policy on $M$ will always follow the shortest path to one of the rewards. To see this, assume there exists a policy $\mu$ that is not following the shortest path from some state $k$ to the next reward-state $k'$. Then, we can improve $\mu$ by taking the shortest path from $k$ to $k'$, contradicting the optimality of $\mu$. The last observation implies that an optimal policy on $M$ is a composition of the local options $\{o _i \}_{i=1}^n$.

Property 3 in Definition \ref{def:2} requires deterministic dynamics. This property is perhaps the most limiting of the three, but again, it appears in numerous domains including many maze navigation problems, the Arcade Learning Environment \citep{bellemare2013arcade}, and games like Chess and Go. Given that $P$ is a deterministic transition matrix, an optimal policy on $M$ will make decisions only at states which contain rewards. In other words, once the policy arrived at a reward state $i$ and decided to go to a reward-state $j$ it will follow the optimal policies $\pi _j$ until it reaches $j$.\footnote{Note that this is not true if $P$ is stochastic. To see this, recall that stochastic shortest path is only shortest in expectation. Thus, stochasticity may lead to states that require changing the decision.} 

For \emph{collectible reward decomposition} (Definition \ref{def:2}), an optimal policy for $M$ can be derived on an SMDP \citep{sutton1999between} denoted by $M_s$. The state space of $M_s$ contains only the initial state $s_0$ and the reward states $\{s_i \}_{i=1}^n$. The action space is replaced by the set of options $\{o_i \}_{i=1}^n$, where $o_i$ corresponds to following the optimal policy in $M _i$ until reaching state $s_i$. In addition, in this action space, the transition matrix $P$ is deterministic since $\forall s,a_i, \; \exists s'$ such that $P^{a_i}_{s,s'}=1$, and otherwise $P^{a_i}_{s,\cdot}=0$. Finally, the reward signal and the discount factor remain the same. 

In general, optimal policies for SMDPs are guaranteed to be optimal in the original MDP only if the SMDP includes both the options and the regular (primitive) actions \citep{sutton1999between}. In a related study, \citet{mann2015approximate} analyzed landmark options, a specific type of options that plan to reach a state on a deterministic MDP. While landmark options are not related to reward decomposition or collectible rewards, they also represent policies that plan to reach a specific state in an MDP. Given a set of landmark options, \citeauthor{mann2015approximate} analyzed the errors from searching the optimal solution on the SMDP (planning with landmark options) instead of searching it in the original MDP (planning with primitive actions). Under Definition \ref{def:2}, and given that all of our policies and value functions are optimal, all these errors equal zero. Thus, the solution to the SMDP is guaranteed to be optimal to the MDP. In addition, the analysis of \citet{mann2015approximate} provides bounds for dealing with sub-optimal options and nondeterministic dynamics that may help to extend our analysis of these cases in future work. 

%{\color{red}{On the one hand, mapping $M$ to $M_s$ allows one to solve $M$ using SMDP algorithms (with reduced state space size); however, it is still intractable to do it. To see this observe that the SMDP still suffers from an exponentially large state space that emerges when trying to model collectible rewards as one-time non-repeatable rewards. On the other hand under, an optimal policy on $M_s$ can be derived by solving an RD-TSP. To see this, first, observe that a TSP policy defines the order to collect the rewards. However, being discounted, there is an accumulative effect on collecting the rewards for the RD-TSP.}}

Finally, an optimal policy on $M_s$ can be derived by solving an RD-TSP (Definition \ref{def:3}). To see this, look at a graph that includes only the initial state and the reward states. Define the length $d_{i,j}$ of an edge $e_{i,j}$ in the graph to be $V_j(i)$, i.e., the value of following option $j$ from state $i$. A path in the graph is defined by a set of indices $\{i_t \}_{t=1}^n$and the length of the path is given by $\sum _{t=0}^{n-1} d_{i_t,i_{t+1}}$. %Definition \ref{def:3} uses this notation to define an optimal solution for the RD-TSP.  %An optimal policy over options can be derived observe that a TSP policy defines the order to collect the rewards. However, being discounted, there is an accumulative effect on collecting the rewards for the RD-TSP. We note that the RD-TSP is a maximization problem over the cumulative discounted reward, while the TSP is a minimization problem over the length of the path. Thus, the approximation factors are in the inverse direction.

\begin{definition}[RD-TSP] 
\label{def:3}
Given an undirected graph with $n$ nodes and edges of length $e_{i,j}$, find a path in the graph that maximizes the discounted cumulative return:
$$
\{i_t^* \}_{t=1}^n  = \underset{\{i_t \}_{t=1}^n \in \text{perm} \{ 1, .., n\} }{\text{arg max}}\sum_{j=0}^{n-1} \gamma ^ {\sum _{t=0}^{j} d_{i_t,i_{t+1}}}
$$
\end{definition}

To summarize, our modeling approach allows us to deal with the curse of dimensionality in three different ways. First, each option can be learned with function approximation techniques, e.g., \cite{tessler2017deep,bacon2017option} to deal with raw, high dimensional inputs like vision and text. Second, formulating the problem as an SMDP reduces the state space to include only the reward states and effectively reduces the planning horizon \citep{sutton1999between,mann2015approximate}. And third, under the RD-TSP formulation, we derive approximate solutions for dealing with the exponentially large state spaces that emerge from modeling one-time events like collectible rewards.

\section{Local heuristics}
\label{sec:local}

%Let $CR$ be the set of the collectible rewards.
%A policy $\pi$ defines a mapping from $CR$ to distributions over options of uncollected rewards, $\pi:CR\rightarrow \Delta(\{o_i\}_{i\in CR})$, where 
% 
%For RD-TSP, the state space and the option space both correspond to the locations of the collectible rewards. Therefore both state and option spaces can be represented as integers, and $\pi$ is a mapping from states to actions $\pi:\{1 .. n \} \rightarrow \Delta \{ 1 .. n\}$, where $\Delta \{ 1 .. n\}$ represents a distribution over actions. For deterministic policies, $\pi:\{1 .. n \} \rightarrow \{ 1 .. n\}$; given a starting node and a set of nodes numbered from $1$ to $n$, a deterministic policy can be represented as a permutation of $\{1 .. n \}$ (representing the path that it produces). Thus, there are $n!$ possible deterministic policies. 
We start by defining a local policy.
All the policies which we analyze are local policies.
A local policy is  a mapping who inputs are:
 (1) the current state $x$, 
 (2) the history $h$ containing the previous steps taken by the policy; in particular $h$ encodes the collectible states that we had already visited, and 
 (3) the discounted return for each reward from the current state, i.e., $\{V_i (x)\}_{i=1}^n$,
 and whose output is a distribution over the options. 
 Notice that
  a local policy  does not have full information on the graph (but only on local distances). Formally,

\begin{definition}[Local policy]
A local policy $\pi_{local}$ is a mapping:
%from a reward state $x$, history $h$, and the option value functions evaluated at $x$, and the history $h$, to a distribution over options,
$$\pi_{local} (x,h,\{V_i (x)\}_{i=1}^n) \rightarrow \Delta (\{o_i\}_{i=1}^n), $$
where $\Delta(X)$ is the set of distributions over a finite set $X$.
\end{definition}
%where the mapping may be either deterministic or stochastic. 
Due to space considerations, most of the proofs are found in the supplementary material.
\subsection{NN performance}
We start with an analysis of one of the most natural heuristics for the TSP, the famous NN algorithm. In the context of our problem, NN is the policy selecting the option with the highest estimated value, exactly like GPI \citep{barreto2016successor}. 
We shall abuse the notation slightly and use the same name (e.g., NN) for the algorithm itself and its value; no confusion will arise.
For TSP (without discount) in general graphs, 
we know that \citep{rosenkrantz1977analysis}:
$$
\frac{1}{3} \log_2(n+1) +\frac{4}{9} \le \frac{\text{NN}}{\text{OPT}} \le \frac{1}{2} \ceil{\log_2(n)}+\frac{1}{2}
$$

However, for  RD-TSP the NN algorithm only guarantees a value of $\frac{OPT}{n}$, as Theorem \ref{theo:1} states. In the next subsection, we prove a lower bound for deterministic local policies (such as NN) of $O \left( \frac{OPT}{n}\right)$, that implies that NN is optimal for deterministic policies. The observation that NN is at least $OPT / n$ of OPT motivated us to use NN as a component in our stochastic algorithms. 

\begin{theorem}[NN Performance]
\label{theo:1}
For any discounted reward traveling salesman graph with n nodes, and any discount factor $\gamma$: {\Large $\enspace \enspace  \frac{\text{NN}}{\text{OPT}} \ge  \frac{1}{n} . $}
\end{theorem}

The basic idea of the proof (in the supplementary material) is that the NN has the first reward larger than any reward that OPT collects.\\
Next, we propose a simple, easy to implement, stochastic adjustment to the vanilla NN algorithm with a better upper bound which we call R-NN (for Random-NN). The algorithm starts by collecting one of the rewards, say $s_1$, at random, and continues by executing NN (Algorithm \ref{alg:nn-random}). 

\begin{algorithm}[h]
\caption{R-NN: NN with a First Random Pick}
\label{alg:nn-random}
\begin{algorithmic}
\STATE {\bfseries Input:} Graph G, with n nodes, and $s_0$ the first node
\STATE Flip a coin 
\IF[Perform NN]{outcome = heads}
    \STATE Visit a node at random, denote it by $s_1$   
\ENDIF
\STATE Follow by executing NN 
\end{algorithmic}
\end{algorithm}

The following theorem shows that our stochastic modification to NN improves its guarantees by a factor of $\log (n).$ 

\begin{theorem}[R-NN Performance]
\label{theo:nn-rand}
For a discounted reward traveling salesman graph with n nodes: {\Large $ \enspace \enspace
\frac{\text{R-NN}}{\text{OPT}} \ge \Omega \left( \frac{\log (n)}{n} \right).$}

\end{theorem}

% \begin{figure*}[h]
% \centering
% \includegraphics[width=\textwidth,, angle =270]{neg.jpg}
% \caption{Negative Examples}
% \label{fig:negative_examples}
% \end{figure*}

%The proof can be found in the supplementary material.
%\footnote{We recommend reading the proof for the NN-RDFS algorithm first.}

While the improvement over NN may seem small ($\log (n)$) the observation that stochasticity improves the performance guarantees of local policies is essential to our work.\\
In the following sections, we derive more sophisticated randomized algorithms with better performance guarantees.
\subsection{Impossibility Results}
\subsubsection{Deterministic Local Policies}
In the previous subsection, we saw that the NN heuristic guarantees performance of at least $\frac{\text{OPT}}{n}.$ Next, we show an impossibility result for all deterministic local policies, indicating that no such policy can guarantee more than $\frac{\text{OPT}}{n},$ which makes NN optimal over such policies. 
\begin{theorem}[Impossibility for Deterministic Local Policies] 
\label{theo:2det}
%For each $n > 4$, and 
For any deterministic local policy D-Local, there exists a  graph with $n$ nodes and a discount factor $\gamma = 1-\frac{1}{n}$ such that: {\Large \enspace \enspace $ \frac{\text{D-Local}}{\text{OPT}} \leq \frac{24}{n}. $}
\end{theorem}

\textit{Proof sketch.} Consider a family of graphs, ${\cal G}$, each of which consists of a star with a central vertex and $n$ leaves at a distance $d$.
The starting vertex is the central vertex, and there is a reward at each leaf. Each graph of the family ${\cal G}$ 
corresponds to a different subset of $n/2$ of the leaves which we connect (pairwise) by edges of length $1\ll d$. (The other $\frac{n}{2}$ leaves are only connected to the central vertex.) While at the central vertex, local policy cannot distinguish among the $n$ rewards (they all at the same distance from the origin), and therefore its choice is the same for all graphs in ${\cal G}$. It follows that, for any given policy, there exists a graph in ${\cal G}$ such that the adjacent $n/2$ rewards are visited last. It follows from simple algebra that $\frac{\text{D-Local}}{\text{OPT}}\le \frac{24}{n}$. $\qed$
%The complete proof can be found in the supplementary material. 

\subsubsection{Stochastic Local Policies}
In the previous subsection, we saw that deterministic local policies could only guarantee $\frac{\text{OPT}}{n}.$ We then showed that NN is optimal over such policies and that a small stochastic adjustment can improve its guarantees. These observations motivated us to look for better local policies in the broader class of stochastic local policies. We begin by providing a better impossibility result for such policies in Theorem \ref{theorem:sto_loc_lower}.

\begin{theorem}[Impossibility for Stochastic Local Policies]
\label{theorem:sto_loc_lower}
%For each $n > 3$, and 
For each stochastic local policy S-Local, there exists a  graph with $n$ nodes and a discount factor $\gamma = 1-\frac{1}{\sqrt{n}}$ such that:{\Large $\enspace \enspace  \frac{\text{S-Local}}{\text{OPT}} \leq \frac{8}{\sqrt{n}}. $}
\end{theorem}

The proof (in the supplementary material) is similar to the previous one but considers a family of graphs where $\sqrt{n}$ leaves are connected to form a clique (instead of $n/2$). 

We do not have a policy that achieves this lower bound, but we now propose and analyze two stochastic policies (in addition to the R-NN) that substantially improve over the deterministic upper bound. As we will see, these policies satisfy the Occam's razor principle, i.e., policies with better guarantees are also more complicated and require more computational resources. 

\subsection{NN with Randomized DFS (RDFS)}
We now describe the NN-RDFS policy (Algorithm \ref{alg:sp-dfs}), the best performing local policy we were able to derive. The policy performs NN with probability $0.5$ and local policy which we call RDFS with probability $0.5$. RDFS starts at a random node and continues by performing a DFS on edges shorter than $\theta$, where $\theta$ is chosen at random as we specify later. When it runs out of edges shorter than $\theta$ then RDFS continues by performing NN. 

\begin{algorithm}[h]
\caption{NN with RDFS}
\label{alg:sp-dfs}
\begin{algorithmic}
\STATE {\bfseries Input:} Graph G, with $n$ nodes, and $s_0$ the first node
\STATE Let $x=\mbox{log}_{\frac{1}{\gamma}}(2)$
\STATE Flip a coin 
\IF[Perform RDFS]{outcome = heads} 
    \STATE Visit a node at random, denote it by $s_1$ 
    \STATE Choose at random $ i \sim \text{Uniform} \{ 1,2, ... ,\mbox{log}_2(n) \} $ 
    \STATE Fix $n' =\frac{n}{2^i} $ and set $\theta = \frac{x}{\sqrt{n'}} $
    \STATE Initiate a DFS from $s_1$ on edges shorter than $\theta$ 
\ENDIF
\STATE Follow by executing NN 
\end{algorithmic}
\end{algorithm}

The performance guarantees for the NN-RDFS method are stated in Theorem \ref{theo:nn-dfs}. The analysis is conducted in three steps. In the first two steps, we assume that OPT achieved a value of $\Omega (n')$ by collecting $n'$ rewards at a segment of length $x \le \mbox{log}_\frac{1}{\gamma} (2)$. The first step considers the case where $n' = \Omega(n)$, and in the second step we remove this requirement and analyze the performance of NN-RDFS for the worst value of $n'$. The third step
considers all the value collected by OPT (not necessarily in a segment of length $x$) and completes the proof. In the second and the third steps we loose two logarithmic factors. One since we use a segment of length $x$ in which OPT collects value of at least $\text{OPT} / \log (n)$, and the second for guessing a good enough approximation for $n'$ (for setting $\theta$).

\begin{theorem}[NN-RDFS Performance]
\label{theo:nn-dfs}
For any instance of RD-TSP with $n$ rewards
$$
\frac{\text{NN-RDFS}}{\text{OPT}} \ge 
\begin{cases} 
\Omega \left( \frac{n^{-\frac{1}{2}}}{ \log^2(n)} \right), &\text{if} \enspace \text{OPT} = \Omega (n). \\
\Omega \left( \frac{n^{-\frac{2}{3}}}{\log^2(n)} \right), &\text{otherwise}. \\
\end{cases} 
$$
\end{theorem}
\begin{proof}
\textbf{Step 1.} Assume that OPT collects   a set $S_\text{OPT}$ of $\alpha n$  rewards for some $0\le \alpha \le 1$, in a segment $p$ of length $x = \mbox{log}_\frac{1}{\gamma} (2)$ (i.e. $x$ is the distance from the first reward to the last reward -- it does not include the distance from the starting point to the first reward). Let $d_{\text{min}},d_{\text{max}}$ the shortest and longest distances from $s_0$ to a reward in $S_\text{OPT}$ respectively. By the triangle inequality, $d_{\text{max}} - d_{\text{min}} \le x.$ 
We further assume  that $OPT \le O(\gamma^{d_\text{min}}\alpha n)$ (i.e., That is the value that OPT collects from rewards which are not in $S_\text{OPT}$ is negligible). We now show that RDFS is $\Omega({\sqrt{n}})$ for $\theta=x/\sqrt{n}$. We start with the following Lemma. 

\begin{lemma}
\label{lemma:theta}
For any path $p$ of length $x$, and $\forall \theta \in [0,x]$, there are less than $\frac{x}{\theta}$ edges in $p$ that are larger than $\theta$. 
\end{lemma}
\textit{Proof.} For contradiction, assume there are more than $\frac{x}{\theta}$ edges longer than $\theta$. The length of $p$ is given by 
\begin{multline*}
\sum_i p_i = \sum _{p_i \le \theta } p_i + \sum _{p_i > \theta} p_i \ge 
\sum _{p_i \le \theta } p_i + \frac{ x}{\theta}\theta > x
\end{multline*}
thus a contradiction to the assumption that the path length is at most $x$.  $\qed$

Lemma \ref{lemma:theta} assures that after pruning all edges larger than $\theta$ (from the graph), there are at most $\frac{x}{\theta}$ Connected Components (CCs) $\{ C_j \}_{j=1}^\frac{x}{\theta}$ in $S_\text{OPT}.$ In addition, it holds that $\sum _{j=1}^{\frac{x}{\theta}} |C_j| = \alpha n,$ and all the edges inside any connected component $C_j$ are shorter than $\theta$. 

Next, we (lower) bound the total gain of RDFS. Say that RDFS starts at a reward in component $C_j$. Then, since all edges in $C_j$ are shorter than $\theta$, it collects either all the rewards in $C_j$, or at least $x/2\theta$ rewards. Thus, RDFS collects  $\Omega \left( \mbox{min} \{ |C_j|, \frac{x}{\theta}\} \right)$ rewards. \\
Notice that the first random step leads RDFS to a vertex in CC $C_j$ with probability $\frac{|C_j|}{n}$. If more than half of rewards are in CCs s.t $|C_j| \ge \frac{x}{\theta},$ then
\begin{align*}
\mbox{RDFS} \ge & \gamma^{d_\text{max}}\sum_{j=1}^{\frac{x}{\theta}} \frac{|C_j|}{n} \cdot \mbox{min}\left\{ |C_j|,\frac{x}{\theta} \right\} \\
\ge & \gamma^{d_\text{max}}\sum_{j: |C_j| \ge \frac{x}{\theta} } \frac{|C_j|}{n} \cdot \frac{x}{\theta} \ge \gamma^{d_\text{max}}\frac{\alpha x}{2 \theta}.
\end{align*}
If more than half of rewards in $S_\text{OPT}$ are in CCs such that $|C_j| \le \frac{x}{\theta},$ let $s$ be the number of such CCs and notice that $s \le \frac{x}{\theta}.$ We get that: 
\begin{align*}
\mbox{RDFS} & = \gamma^{d_\text{max}} \sum_{j=1}^{\frac{x}{\theta}} \frac{|C_j|}{n} \cdot \mbox{min}\left\{ |C_j|,\frac{x}{\theta} \right\}  \\ 
& \ge  \gamma^{d_\text{max}} \sum_{j: |C_j| \le \frac{x}{\theta} } \frac{|C_j|^2}{n} \ge \frac{s}{n} \gamma^{d_\text{max}} \left( \frac{1}{s} \sum_{j=1}^s  |C_j|^2 \right)  \\ 
& \underset{\mbox{Jensen}}{\ge}  \frac{s}{n} \gamma^{d_\text{max}}  \left( \frac{1}{s}\sum_{j=1}^s |C_j| \right) ^2 \ge  \gamma^{d_\text{max}} \frac{\theta \alpha^2 n}{4x}. 
\end{align*}

By setting $\theta = \frac{x}{\sqrt{n}}$ we guarantee that the value of RDFS is at least $\gamma^{d_\text{max}}\alpha^2 \sqrt{n} / 4$. Since $d_\text{max}-d_\text{min}\leq x$,
$$
\frac{\text{RDFS}}{\text{OPT}} \ge \frac{ \gamma^{d_\text{max}} \alpha^2 \sqrt{n} / 4}{\gamma^{d_\text{min}} \alpha n} \ge \frac{\alpha \gamma ^{x}}{4\sqrt{n}} = \frac{\alpha}{2\sqrt{n}} ,
$$ where the last inequality follows from the triangle inequality. 

\textbf{Step 2.} 
Assume that OPT gets its value from $n'<n$ rewards that it collects in a segment of length $x$ (and from all other rewards OPT collects a negligible value). Recall that the NN-RDFS policy is either NN with probability $0.5$ or RDFS with probability $0.5$. By picking the single reward closest to the starting point, NN gets at least $1/n'$ of the value of OPT. Otherwise, with probability $n'/n$, RDFS starts with one of the $n'$ rewards picked by OPT and then, by the analysis of step 1, if it sets $\theta = \frac{x}{\sqrt{n'}}$, RDFS collects  $\frac{1}{2\sqrt{n'}} $  of the value collected by OPT (we use Step (1) with $\alpha =1$). It follows that $$
\frac{\text{NN-RDFS}}{\text{OPT}} \ge \frac{1}{2} \cdot \frac{1}{n'} + \frac{1}{2} \cdot \frac{n'}{n} \cdot \frac{1}{2\sqrt{n'}}  = \frac{1}{2n'} + \frac{\sqrt{n'}}{4n}.
$$

This lower bound is smallest when $n' \approx n^{\frac{2}{3}}$, in which case NN-RDFS collect $\Omega(n^{-2/3})$ of OPT.  

First, notice that since $n'$ is not known to NN-RDFS, it has to be guessed in order to choose $\theta$. This is done by setting $n'$ at random from $n' = n / 2 ^i, i \sim \text{Uniform} \{ 1,2, ..., \mbox{log}_2(n) \} $. This guarantees that with probability $\frac{1}{\log(n)}$ our guess for $n'$ will be off of its true value by a factor of at most $2$. With this guess we will work with an approximation of $\theta$ which is off of its true value by a factor of at most ${\sqrt{2}}$. These approximations degrade our bounds by a factor of $\log (n).$

\textbf{Step 3.}
Finally, we consider the general case where OPT may collect its value in a segment of length larger than $x$. Notice that the value which OPT collects from rewards that follow the first $\log_2 (n)$ segments of length $x$ in its tour is at  most $1$ (since $\gamma ^{\log_{2} (n)\cdot x} = \frac{1}{n}$). This means that there exists at least one segment of length $x$ in which OPT collects at least  $\frac{\text{OPT}}{\log_{2}(n)}$ of its value. Combining this with the analysis in the previous step, the proof is complete. % we get that
% $$
% \frac{\text{NN-RDFS}}{\text{OPT}} \ge 
% \begin{cases} 
% \Omega \left( \frac{n^{-\frac{1}{2}}}{ \mbox{log}(n)^2} \right), &\text{if} \enspace \text{OPT} = \Omega (n). \\
% \Omega \left( \frac{n^{-\frac{2}{3}}}{\mbox{log}(n)^2} \right), &\text{otherwise}. \\
% \end{cases} 
% $$
\end{proof}

\subsection{NN with a Random Ascent (RA)}
We now describe the NN-RA policy (Algorithm \ref{alg:nn-bfs}). Similar in spirit to NN-RDFS, the policy performs NN with probability $0.5$ and local policy which we call RA with probability $0.5$. RA starts at a random node, $s_1$, sorts the nodes in increasing order of their distance from $s_1$ and then visits all other nodes in this order. The algorithm is simple to implement, as it does not require guessing any parameters (like $\theta$ which RDFS has to guess). However, this comes at the cost of a worse worst case bound.

\begin{algorithm}[h]
\caption{NN with RA}
\label{alg:nn-bfs}
\begin{algorithmic}
\STATE {\bfseries Input:} Graph G, with n nodes, and $s_0$ the first node
\STATE Flip a coin, 
\IF[Perform RA]{outcome = heads}
    \STATE Visit a node at random, denote it by $s_1$
    \STATE Sort the remaining nodes by increasing distances from $s_1$, call this permutation $\pi$ 
    \STATE Visit the nodes in $\pi$ in increasing order 
\ELSE
\STATE  Execute NN 
\ENDIF
\end{algorithmic}
\end{algorithm}

The performance guarantees for the NN-RA method are given in Theorem \ref{theo:nn-bfs}. The analysis follows the same steps as the proof of the NN-RDFS algorithm.
%and can found in the supplementary material. 
We emphasize that here, the pruning parameter $\theta$ is only used for analysis purposes and is not part of the algorithm. 
Consequently, we see only one logarithmic factor in the performance bound of Theorem~\ref{theo:nn-bfs} in contrast with two in Theorem~\ref{theo:nn-dfs}.

\begin{theorem}[NN-RA Performance]
\label{theo:nn-bfs}
For a discounted reward traveling salesman graph with $n$ nodes and for any discount factor $\gamma$:
$$
\frac{\text{NN-RA}}{\text{OPT}} \ge 
\begin{cases} 
\Omega \left( \frac{n^{-\frac{2}{3}}}{ \log(n)} \right), &\text{if} \enspace \text{OPT} = \Omega (n). \\
\Omega \left( \frac{n^{-\frac{3}{4}}}{\log(n)} \right), &\text{otherwise}. \\
\end{cases} 
$$
\end{theorem}

\begin{figure*}[b]
\centering
\includegraphics[width=\textwidth]{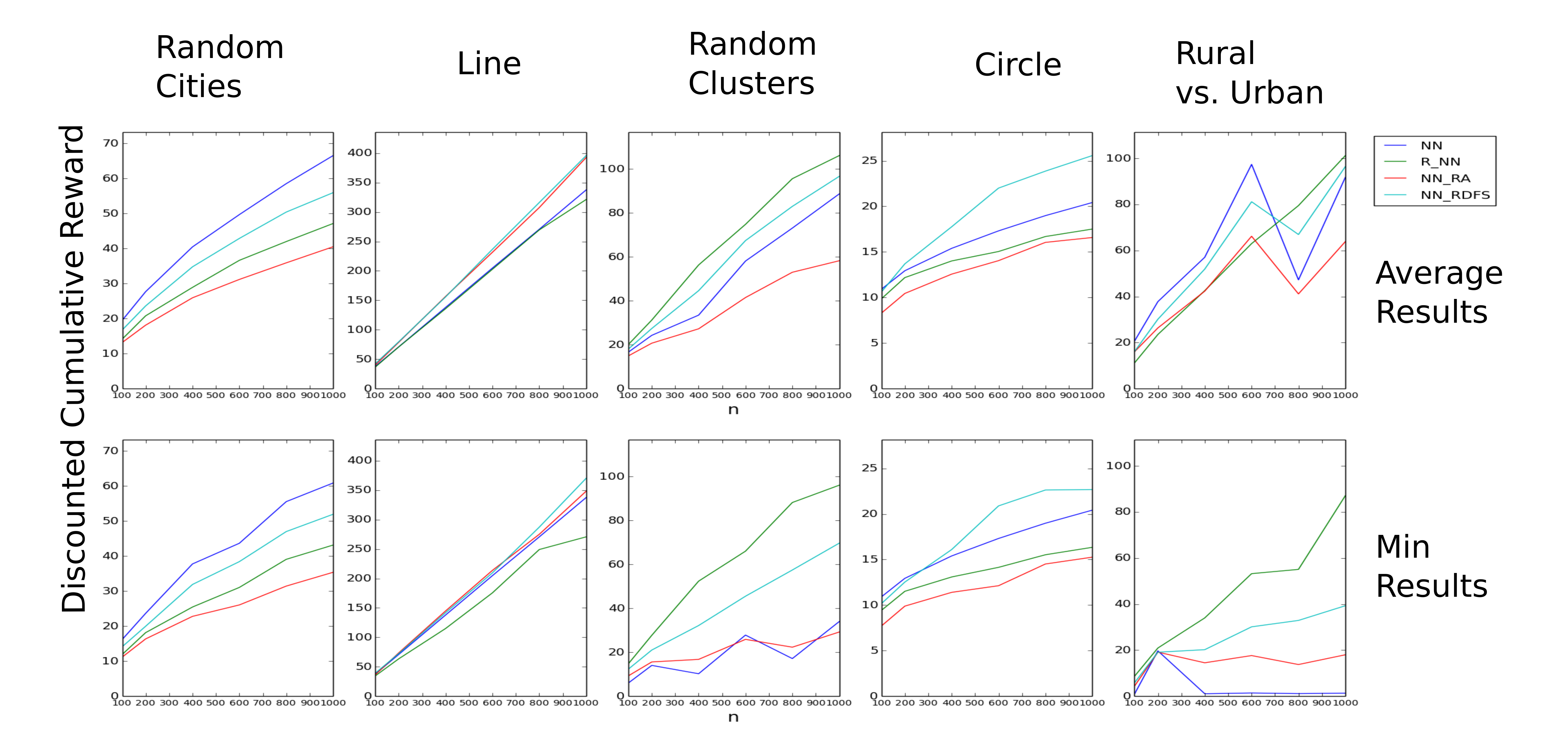}
\caption{Evaluation of deterministic and stochastic local policies over different RD-TSPs. The cumulative discounted reward of each policy is reported for the average and worst case scenarios.}
\label{fig:sims}
\end{figure*}
\section{Simulations}
 
In this section, we evaluate and compare the performance of deterministic and stochastic local policies by measuring the (cumulative discounted) reward achieved by each algorithm on different RD-TSP instances as a function of $n,$ the number of the rewards, with $n\in \{100,200,400,600,800,1000\}.$ For each $n$, we set $\gamma = 1 - \frac{1}{n}$ and 
place the rewards such that OPT can collect almost all of them within a constant discount
 (i.e., $OPT\approx\alpha n$). We always place the initial state $s_0$ at the origin, i.e., $s_0=(0,0)$.
We define $x = \log_{\frac{1}{\gamma}}(2)$, 
and  $\ell=0.01x$ denotes a short distance. Next, we describe five scenarios (Figure \ref{fig:sims}, ordered from left to right) that we considered for evaluation. For each of these graph types, we generate $N_{maps}=10$ different graphs, and report the reward achieved by each algorithm on average over the $N_{maps}$ graphs (Figure \ref{fig:sims}, Top), and in the worst-case (the minimal among the $N_{maps}$ scores) (Figure \ref{fig:sims}, Bottom). As some of our algorithms are stochastic, we report average results, i.e., for each graph we run each algorithm $N_{alg}=100$ times and report the average score. Finally, we provide a visualization of the different graphs and the tours taken by the different algorithms, which helps in understanding our numerical results.  

\textbf{(1) Random Cities.} For a vanilla TSP with $n$ rewards randomly distributed on a $2D$ plane, it is known that the NN algorithm yields a tour which is  $25\%$ longer than optimal on average \cite{johnson1997traveling}.
We used a similar input to compare our algorithms for
 RD-TSP, specifically, we generated a graph with $n$ rewards $r_i\sim(\text{U}(0,x),\text{U}(0,x)),$ where U is the uniform distribution. 
 
Inspecting Figure \ref{fig:sims}, we can see that the NN algorithm performs the best both on the average and in the worst case. This observation suggests that when the rewards are distributed at random, selecting the nearest reward is the most reasonable thing to do. In addition, we can see that NN-RDFS performs the best among the stochastic policies (as predicted by our theoretical results). On the other hand, the RA policy performs the worst among stochastic policies. This happens because sorting the rewards by their distances from $s_1$,  introduces an undesired ``zig-zag'' behavior while collecting rewards at equal distance from $s_1$. 

\textbf{(2) Line.} This graph demonstrates a scenario where greedy algorithms like NN and R-NN are likely to fail. The rewards are located in three different groups; each contains $n/3$ of the rewards. In group 1, the rewards are located in a cluster left to the origin $r_i\sim(\text{U}[-\theta/3-\ell,-\theta/3+\ell],\text{N}(0,\ell)),$ while in group 2 they are located in a cluster right to the origin $r_i\sim(\text{U}[\theta/3-3\ell,\theta/3-2\ell],\text{N}(0,\ell))$ but a bit closer than group 1 ($\theta=\frac{x}{\sqrt{n}}$). Group 3 is also located to the right, but the rewards are placed in increasing distances, such that the $i$-th reward is located at $(\theta/3)2^i$. 

Inspecting the results, we can see that NN and R-NN indeed perform the worst. To understand this, consider the tour that each algorithm takes. NN goes to group 2, then 3 then 1 (and loses a lot from going to 3). The stochastic tours depend on the choice of $s_1$. If it belongs to group 1, they collect group1 then 2 then 3, from left to right, and perform relatively the same. If it belongs to group 3, they will first collect the rewards to the left of $s_1$ in ascending order and then come back to collect the remaining rewards to the right, performing relatively the same. However, if $s_1$ is in group 2, then NN-RDFS, NN-RA will visit group 1 before going to 3, while R-NN is tempted to go to group 3 before going to 1 (and loses a lot from doing it).

\textbf{(3) Random Clusters.} This graph demonstrates the advantage of stochastic policies.
We first randomly place $k=10$ cluster centers $c^j$, $j=1,\ldots,k$ on a circle of radius $x$ (the centers are at a small random (Guassian) distance from the circle)
%{\bf HK is it important that they are not exactly on the circle ? omit thsi detail if not ?}
Then to draw a reward $r_i$ we first draw a cluster center $c^j$ uniformly and then draw $r_i$ such that
$ r_i \sim (\text{U}[c^j_x-10\ell,c^j_x+10\ell],\text{U}[c^j_y-10\ell,c^j_y+10 \ell])$.
 This scenario is motivated by maze navigation problems, where collectible rewards are located at rooms (clusters) while in between rooms there are fewer rewards to collect.
 
Inspecting the results, we can see that NN-RDFS and R-NN perform the best, in particular in the worst case scenario. The reason for this is that NN picks the nearest reward, and most of its value comes from rewards collected at this cluster. On the other hand, the stochastic algorithms visit larger clusters first with higher probability and achieve higher value by doing so.

 \textbf{(4) Circle.} In this graph, there are $\sqrt{n}$ circles, all centered at the origin, 
 and the radii of the $i$th circle is $\rho_i = \frac{x}{\sqrt{n}}\cdot (1+\frac{1}{\sqrt[4]{n}})^i$
 On each circle we place 
 $\sqrt{n}$ rewards are placed at equal distances.
 
Here, NN-RDFS performs the best among all policies since it collects rewards closer to $s_1$ first. The greedy algorithms, on the other hand, are ``tempted'' to collect rewards that take them to the outer circles which results in lower values. 

\textbf{(5) Rural vs. Urban.} Here, the rewards are sampled from a mixture of two normal distributions. Half of the rewards are located in a ``city'', i.e., their position is a Gaussian random variable with a small standard deviation s.t. $r_i\sim(\text{N}(x,\ell),\text{N}(0,\ell))$. The other half is located in a ``village'', i.e., their position is a Gaussian random variable s.t. $r_i\sim(\text{N}(-x,10x),\text{N}(0,10x)).$ 

In this graph, we can see that in the worst case scenario, the stochastic policies perform much better than NN. This happens because NN is mistakenly choosing rewards that take it to remote places in the rural area, while the stochastic algorithms remain near the city with high probability and collect its rewards.

%To understand why this result is not trivial, consider a related scenario where all the rewards from the village are located to the left and all the rewards from the city are located to the right (up to some negligible factor). {\bf HK left and right with respect to what ? } Next, notice that $\mathbb{E}\sum \gamma^x \ge \sum \gamma ^ {\mathbb{E}x}$ (Jansen), and therefore OPT prefers to collect rewards that are more spread than dense around the mean (larger standard deviation). Since the rewards from the village are more spread, but becoming denser towards the village, in high probability NN will drift towards the rural area, similar to OPT. Therefore, in simulation, we increased the std of the village to be $10x$ such that many rewards from the rural area are in the city half space. On the one hand, NN will still drift towards the village, but on the other hand, if a policy arrives close enough to the city, it should collect its rewards first, therefore, there is a non-trivial trade-off.
%{\bf HK I could not understand this paragraph? some statements seem to hold only on the line ? removing this is also an option}

\subsection{Visualization}

%{\bf [[YM: I think NN-RA for Random Clusters the best and worse plots are switched]]}
\begin{figure*}
\centering
\includegraphics[width=\textwidth]{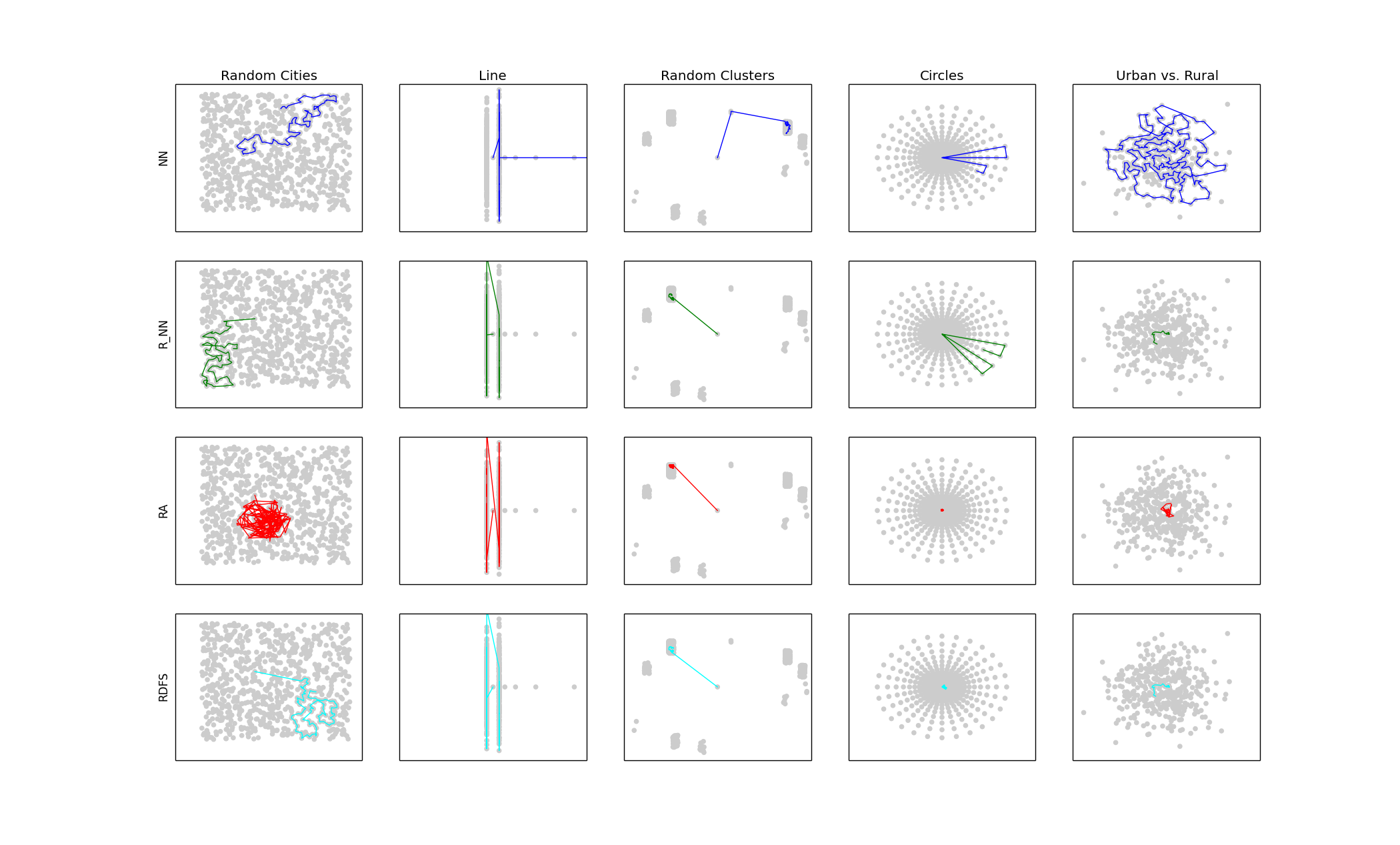}
\caption{Visualization of the best tours, taken by deterministic and stochastic local policies over different RD-TSPs.}
\label{fig:best}
\end{figure*}

\begin{figure*}
\centering
\includegraphics[width=\textwidth]{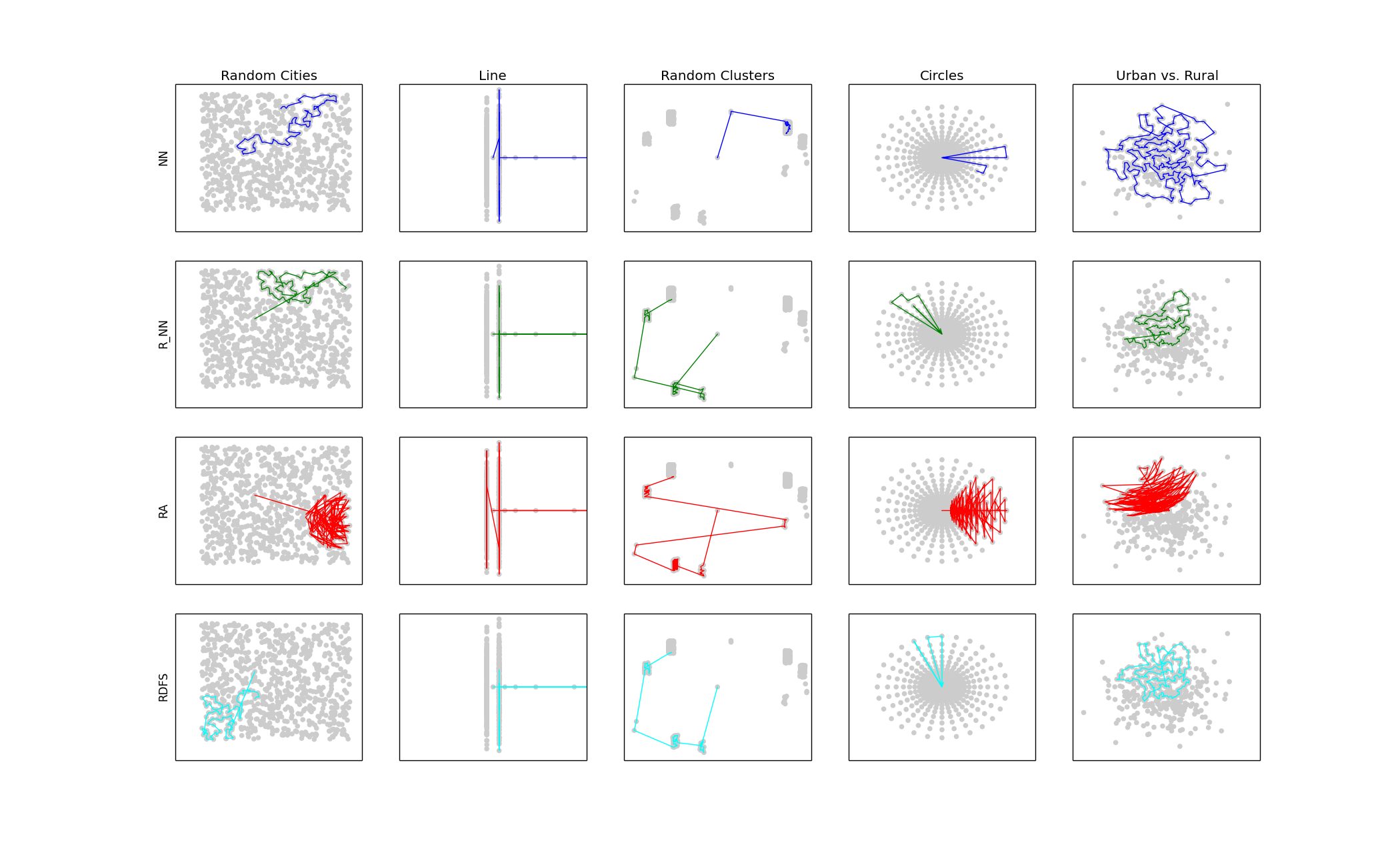}
\caption{Visualization of the worst tours, taken by deterministic and stochastic local policies over different RD-TSPs.}
\label{fig:worst}
\end{figure*}

We now present a visualization of the tours taken by the different algorithms. Note that in order to distinguish between the algorithms qualitatively, we compare the NN algorithm with the stochastic algorithms RDFS and RA (and not of NN-RDFS and NN-RA, i.e., without balancing them with NN). All the graphs we present have $n=800$ rewards, displayed on a 2D grid using gray dots. For each graph type, we present a single graph sampled from the appropriate distribution.  On top of it, we display the tours taken by the different algorithms, such that each row corresponds to a single algorithm. For the stochastic algorithms, we present the best (Figure \ref{fig:best}) and the worst tours (Figure \ref{fig:worst}), among 20 different runs (for NN we display the same tour since it is deterministic). Finally, for better interpretability, we only display the first $n/k$ rewards of each tour, in which the policy collects most of its value, with $k=8$ ( $n/k=100$ ) unless mentioned otherwise.% {\bf [[YM: Not sure what is the 8-th fraction? The first $n/k=100$ rewards?]]}

\textbf{Discussion.}\\
\textbf{Random cities:} Inspecting the worst tours, we can see that the stochastic tours are longer than the NN tour due to the distance to the first reward (that is drawn at random). In addition, we observe a ``zig-zag'' behavior in the tour of NN-RA while collecting rewards at equal distanced from $s_1$ (but which are not close to each other), which causes it to perform the worst in this scenario. The best tours exhibit similar behavior, but in this case, $s_1$ is located closer to $s_0$.  

\textbf{Line:} Recall that in this graph, the rewards are located in three different groups; each containing $n/3$ of the rewards. Ordered from left to right, group 1 is a cluster of rewards located left to the origin at distance $\theta / 3.$ Group 2 is also a cluster, but located to the right,  in a slightly shorter distance from the origin than group 1. Group 3 is also located to the right, but the rewards are placed in increasing distances, such that the $i$-th reward is located at $(\theta/3)2^i$. 

For visualization purposes, we added a small variance in the locations of the rewards at groups 1 and 2 and rescaled the axes. The two vertical lines of rewards represent these two groups, while we cropped the graph such that only the first few rewards in group 3 are observed. Finally, we chose $k=2$, such the first half of the tour is displayed, and we can see the first two groups visited in each tour. %{\bf [[YM: This is the only place where $k\neq 8$?]]}

Examining the best tours (Figure \ref{fig:best}), we can see that NN first visits group 2, but is then tempted going right, to group 3, which harms its performance. On the other hand, the stochastic algorithms are staying closer to the origin and collect the rewards in groups 1 and 2 first. In the case of the worst case tours, the algorithms perform relatively the same.

\textbf{Random Clusters:} Here, we can see that NN  first visits the cluster nearest to the origin. The nearest cluster is not necessarily the largest one, and in practice, NN collects all the rewards in this cluster and traverses between the remaining clusters which result in lower performance. On the other hand, since the stochastic algorithms are selecting the first reward at random, with high probability the reach larger clusters and achieve higher performance. 

\textbf{Circles:} In this scenario, the distance between adjacent rewards on the same circle is longer than the distance between adjacent rewards on two consecutive circles.  Examining the tours, we can see that indeed NN and R-NN are taking tours that lead them to the outer circles. On the other hand, RDFS and RA are staying closer to the origin. Such local behavior is beneficial for RDFS, which achieves the best performance in this scenario. However, while RA performs well in the best case, its performance is much worse than the other algorithms in the worst case. Hence, its average performance is the worst in this scenario. 

\textbf{Rural vs. Urban:} In this graph, we have a large rural area, where the city is located near the origin which is hard to visualize. To improve the visualization here, we chose $k=2$, such the first half of the tour is displayed. Since half of the rewards belong to the city, choosing $k=2$ ensures that any tour that is reaching the city only the first segment of the tour (until the tour reaches the city) will be displayed. %In practice, all the tours have reached the city in the first half of the tour, but after traversing different distances.

Looking at the best tours, we can see that NN is taking the longest tour before it reaches the city, while the stochastic algorithms reach it earlier. By doing so, the stochastic algorithms can collect many rewards by traversing short distances and therefore perform much better in this scenario.

\section{Related work}
\label{sec:related}
% multi-task decomposition
 {\bf Pre-defined rules for option selection} are used in several studies. \citeauthor{karlsson1997learning} (\citeyear{karlsson1997learning}) suggested a policy that chooses greedily with respect to the sum of the local Q-values $\mbox{a}^*={\mbox{argmax}_a} \sum _i Q_i(s,a)$. \citeauthor{humphrys1996action} (\citeyear{humphrys1996action}) suggested to choose the option with the highest local Q-value $\mbox{a}^*={\mbox{argmax}_{a,i}} Q_i(s,a)$ (NN).
Such greedy combination of local policies that were optimized separately may not necessarily perform well. \citet{barreto2016successor} considered a transfer framework similar to ours (Definition \ref{def:1}), but did not focus on collectible reward decomposition (Definition \ref{def:2}). Instead, they proposed a framework where the rewards are linear in some reward features $R_i(s,a) = w_i ^T \phi (s,a)$.\footnote{Notice that Definition \ref{def:2} is a special case of the framework considered by \citeauthor{barreto2016successor}} Similar to \citep{humphrys1996action}, they suggested using NN as the pre-defined rule for option selection (but referred to it as GPI). In addition, the authors provided performance guarantees for using GPI in the form of additive (based on regret) error bounds but did not provide impossibility results. In contrast, we prove multiplicative performance guarantees for NN, as well as for three stochastic policies.
We also proved, for the first time, impossibility results for such local option selection methods.

A different approach to tackle these challenges is \textbf{Multi-task learning,} in which we optimize the options in parallel with the policy over options \citep{russell2003q,sprague2003multiple,van2017hybrid}. %The key observation here is that the options can be learned locally (concerning local rewards) concerning the composite task.
One method that achieves that goal is the local SARSA algorithm \cite{russell2003q,sprague2003multiple}. Similar to \cite{karlsson1997learning}, a Q function is learned locally for each option (concerning a local reward). However, here the local Q functions are learnt on-policy (using SARSA) with respect to the policy over options $\pi(s)={\mbox{argmax}_a} \sum _i Q_i(s,a),$ instead of being learned off-policy with Q learning. \citet{russell2003q} showed that if the policy over options is being updated in parallel with the local SARSA updates, then the local SARSA algorithm is promised to converge to the optimal value function.

\section{Conclusions}
In this work, we provided theoretical guarantees for reward decomposition in deterministic MDPs, which allows one to learn many policies in parallel and combine them into a composite solution efficiently and safely. In particular, we focused on approximate solutions that are local, and therefore, easy to implement and do not require many computational resources. Local deterministic policies, like Nearest Neighbor, are being used in practice for hierarchical reinforcement learning. Our study provides an important theoretical guarantee on the reward collected by three such policies, as well as impossibility results for any local policy. These policies outperform NN in the worst case scenario; we evaluated them in the average, and worst case scenarios, suggesting when is it better to use each one.

%\citeauthor{barreto2016successor}, \citep{barreto2016successor} proposed a similar transfer framework to ours where reward functions are linear in some reward features $R_i(s,a) = w_i ^T \phi (s,a)$, and the features are fixed over a set of tasks. It is easy to see that collectible rewards satisfy these conditions in the case of tabular features. They proposed two mechanisms to allow transfer: (1) Successor representations \citep{dayan1993improving}, a value function representation that decouples the dynamics of the environment from the rewards, and allows to learn policy-specific features that are transferable between different MDPs. Learning successor representations allows one to evaluate a policy that was learned for MDP with reward signal ($w_i$) on a different MDP with reward signal ($w_j$) immediately (zero-shot transfer). (2) Greedy Policy Improvement (GPI), a method to select between previously trained policies by selecting the policy with the maximal value prediction. The authors provide bounds for using GPI with SF based on the distance between the weights of the rewards seen in the past and the reward evaluated on transfer. 

\clearpage
\medskip
\small
\bibliography{paper_bib}
\bibliographystyle{icml2018}
%\end{document}

\clearpage
\section{NN performance}
\begin{theorem*}[NN Performance]
For any discounted reward traveling salesman graph with n nodes, and any discount factor $\gamma$: $$ \frac{\text{NN}}{\text{OPT}} \ge  \frac{1}{n} . $$
\end{theorem*}

\begin{proof}
Denote by $i^*$ the nearest reward to the origin $s_0$, and by $d_{0,i^*}$ the distance from the origin to $i^*$.
The distance from $s_0$ to the first reward collected by OPT is at least  $d_{0,i^*}$.
 Thus, if $o_0=s_0,o_1,\ldots,o_{n-1}$ are the rewards ordered in the order by which OPT collects them we get that
\begin{align*}
\text{OPT} &= \sum_{j=0}^{n-1} \gamma ^ {\sum _{t=0}^{j} d_{o_t,o_{t+1}}} \\
&\leq \gamma ^{d_{0,i^*}} \left(1+ \sum_{j=1}^{n-1} \gamma ^ {\sum _{t=1}^{j} d_{o_t,o_{t+1}}} \right) \leq n \gamma ^ {d_{0,i^*}} 
\end{align*}
On the other hand, the NN heuristic chooses $i^*$ in the first round, thus, its cumulative reward is at least $\gamma ^{d_{0,i^*}}$ and we get that
$$ \frac{\text{NN}}{\text{OPT}} \ge \frac{\gamma ^{d_{0,i^*}}}{n\gamma^{d_{0,i^*}} } = \frac{1}{n}. $$
\end{proof}

\section{Impossibility results}
\subsection{Deterministic Local Policies}
\begin{theorem*}[Impossibility Results for Deterministic Local Policies] 
%For each $n > 4$, and 
For any deterministic local policy D-Local, there exists a  graph with $n$ nodes and a discount factor $\gamma = 1-\frac{1}{n}$ such that $$ \frac{\text{D-Local}}{\text{OPT}} \leq \frac{24}{n}. $$
\end{theorem*}

\begin{proof}
Consider a family of graphs, ${\cal G}$, each of which consists of a star 
with a central vertex and $n$ leaves.
The starting vertex is the central vertex, and there is a reward at each leaf. The length of each edge is $d$, where $d$ is
chosen s.t.
 $\gamma^d = \frac{1}{2}$.

Each graph of the family ${\cal G}$ 
corresponds to a different subset of $n/2$ of the leaves which we connect (pairwise) by edges of length $1$. (The
 other $\frac{n}{2}$ leaves are only connected to the central vertex.) 
While at the central vertex, local policy cannot distinguish among the $n$ rewards (they all at the same distance from the origin), and therefore its choice is the same for all graphs in ${\cal G}$. (The following choice is also the same and so on, as long as it does not hit one of the $n/2$ special rewards.)

It follows that, for any given policy, there exists a graph 
in ${\cal G}$ such that the adjacent $n/2$ rewards are visited last. Finally, since $\gamma = 1-\frac{1}{n}$ we have that 
$\frac{n}{4} \le \sum_{i=0} ^{\frac{n}{2}-1} \gamma ^i = \frac{1-\gamma^{n/2}}{1-\gamma} \le \frac{n}{2}$ and thus
\begin{align*}
\frac{\text{D-Local}}{\text{OPT}} &= \frac{\sum_{i=1} ^{\frac{n}{2}} \gamma ^ {(2i-1)d} +  \gamma ^ {nd+1} \sum_{i=0} ^{\frac{n}{2}-1} \gamma ^i}{\gamma ^ {d} \sum_{i=0} ^{\frac{n}{2}-1} \gamma ^i + \gamma ^{2d+\frac{n}{2}-1} \sum_{i=1} ^{\frac{n}{2}} \gamma ^ {(2i-1)d}} \\
&\le \frac{\sum_{i=1} ^{\frac{n}{2}} \gamma ^ {(2i-1)d} +  0.5n \gamma ^ {nd+1} }{\gamma ^ {d} \sum_{i=0} ^{\frac{n}{2}-1} \gamma ^i } \\
&= \frac{2 \sum_{i=1} ^{\frac{n}{2}} 0.25 ^ i +  \frac{0.5 ^ {n} n}{4}  }{0.5 \sum_{i=0} ^{\frac{n}{2}-1} \gamma ^i } \le 
\frac{6}{\sum_{i=0} ^{\frac{n}{2}-1} \gamma ^i } \le \frac{24}{n}.
\end{align*}

\end{proof}
\subsection{Stochastic Local Policies}

\begin{theorem*}[Impossibility Results for Stochastic Local Policies]
%For each $n > 3$, and 
For each stochastic local policy S-Local, there exists a  graph with $n$ nodes and a discount factor $\gamma = 1-\frac{1}{\sqrt{n}}$ such that $$ \frac{\text{S-Local}}{\text{OPT}} \leq \frac{8}{\sqrt{n}}. $$
\end{theorem*}

\begin{proof}
We consider a family of graphs, ${\cal G}$, each of which consists of a star
with a central vertex and $n$ leaves.
The starting vertex is the central vertex, and there is a reward at each leaf. The length of each edge is $d$, where $d$ is
chosen such that 
 $\gamma^d = \frac{1}{2}$.
 Each graph in ${\cal G}$ corresponds to a subset of 
 $\sqrt{n}$ leaves which we pairwise connect to form a clique.

Since $\gamma = 1-\frac{1}{\sqrt{n}},$ we have that $\sum_{i=0} ^{\sqrt{n}-1} \gamma ^i \ge \frac{\sqrt{n}}{2},$ and therefore
\begin{align*}
\text{OPT}&=\gamma^d \sum _{i=0}^{\sqrt{n}-1} \gamma ^i + \gamma ^{2d+\sqrt{n}-1} \sum_{i=1} ^{n-\sqrt{n}} \gamma ^ {(2i-1)d} \\
& \geq 0.5 \sum _{i=0}^{\sqrt{n}-1}\gamma ^ i \ge 0.25 \sqrt{n}.
\end{align*}

On the other hand, local policy at the central vertex cannot distinguish among the rewards and therefore for every graph in ${\cal G}$ it picks the first reward from the same distribution. The policy continues to choose rewards from the same distribution until it hits the first reward from the $\sqrt{n}$-size clique. 

To argue formally that every S-Local policy has small expected reward on a graph from ${\cal G}$, we use Yao's principle \cite{yao1977probabilistic} and consider the expected reward of a D-Local policy on the uniform distribution over ${\cal G}$. 

Let $p_1 = \sqrt{n}/n$ be the probability that D-Local picks its first vertex from the $\sqrt{n}$-size clique. 
Assuming that the first vertex is not in the clique, let $p_2 = \sqrt{n}/(n-1)$ be the probability that the second vertex is from the clique, and let $p_3$, $p_4$, $\ldots$ be defined similarly. When D-local picks a vertex in the clique then its reward (without the cumulative discount) is $O(\sqrt{n})$. However, each time D-Local misses the clique then it collects a single reward but suffers a discount of
$\gamma^{2d} = 1/4$. 
Neglecting the rewards collected until it hits the clique, the total value of D-Local is 
$$
O\left( \left( p_1 + (1-p_1)\gamma^{2d}p_2 + (1-p_1)(1-p_2)\gamma^{4d}p_3 \ldots\right)  \sqrt{n}\right)
$$
Since $p_i \le 2/\sqrt{n}$ for $1\le i\le n/2$ this value is $O(1)$
\end{proof}

\section{NN-RA}
\begin{theorem*}[NN-RA Performance]
For any RD-TSP instance with $n$ nodes and for any discount factor $\gamma$
$$
\frac{\text{NN-RA}}{\text{OPT}} \ge 
\begin{cases} 
\Omega \left( \frac{n^{-\frac{2}{3}}}{ \log(n)} \right), &\text{if} \enspace \text{OPT} = \Omega (n). \\
\Omega \left( \frac{n^{-\frac{3}{4}}}{\log(n)} \right), &\text{otherwise}. \\
\end{cases} 
$$
\end{theorem*}
\begin{proof}
\textbf{Step 1.}
Assume that OPT collects   a set $S_\text{OPT}$ of $\alpha n$  rewards for some $0\le \alpha \le 1$, in a segment $p$ of length $x = \mbox{log}_\frac{1}{\gamma} (2)$ (i.e. $x$ is the distance from the first reward to the last reward -- it does not include the distance from the starting point to the first reward). Let $d_{\text{min}},d_{\text{max}}$ the shortest and longest distances from $s_0$ to a reward in $S_\text{OPT}$ respectively. By the triangle inequality, $d_{\text{max}} - d_{\text{min}} \le x.$ 
We further assume  that $OPT \le O(\gamma^{d_\text{min}}\alpha n)$ (i.e., That is the value that OPT collects from rewards which are not in $S_\text{OPT}$ is negligible). 

Let $\theta$ be a threshold that we will fix below, and denote by $\{C_j\}$ the CCs of $S_\text{OPT}$ that are created by deleting edges longer than $\theta$ among vertices of $S_\text{OPT}$. By Lemma  \ref{lemma:theta}, we 
 have at most $x/\theta$ CC.
 
 Assume that RA starts at a vertex of a component $C_j$, such that $|C_j|=k$.
 Since the diameter of $C_j$ is at most $(|C_j|-1)\theta$ then it collects its first $k$ vertices (including $s_1$) within a total distance 
 of $2\sum _{i=2}^{k} (i-1)\theta \le k^2 \theta $. 
 So if $k^2 \theta \le x$ then it collects at least $|C_j|$ rewards before traveling a total distance of $x$, and if 
 $k^2 \theta > x$ it collects at least $\lfloor \sqrt{x/\theta} \rfloor$ rewards. (We shall omit the floor function for brevity in the sequal.) 
 It follows that RA collects  $\Omega \left( \mbox{min} \{ |C_j|, \sqrt{\frac{x}{\theta}}\} \right)$ rewards.
Notice that the first random step leads RDFS to a vertex in CC $C_j$ with probability $\frac{|C_j|}{n}$. If more than half of rewards are in CCs s.t $|C_j| \ge \sqrt{\frac{x}{\theta}},$ then
\begin{align*}
\mbox{RA} \ge & \gamma^{d_\text{max}}\sum_{j=1}^{\frac{x}{\theta}} \frac{|C_j|}{n} \cdot \mbox{min}\left\{ |C_j|,\sqrt{\frac{x}{\theta}} \right\} \\
\ge & \gamma^{d_\text{max}}\sum_{j: |C_j| \ge \sqrt{\frac{x}{\theta}} } \frac{|C_j|}{n} \cdot \sqrt{\frac{x}{\theta}} \ge \gamma^{d_\text{max}}\frac{\alpha }{2 }  \sqrt{\frac{x}{\theta}}.
\end{align*}
If more than half of rewards in $S_\text{OPT}$ are in CCs such that $|C_j| \le \sqrt{ \frac{x}{\theta}},$ let $s$ be the number of such CCs and notice that $s \le \frac{x}{\theta}.$ We get that: 
\begin{align*}
\mbox{RA} & = \gamma^{d_\text{max}} \sum_{j=1}^{\frac{x}{\theta}} \frac{|C_j|}{n} \cdot \mbox{min}\left\{ |C_j|,\frac{x}{\theta} \right\}  \\ 
& \ge  \gamma^{d_\text{max}} \sum_{j: |C_j| \le \sqrt{ \frac{x}{\theta} } } \frac{|C_j|^2}{n} \ge \frac{s}{n} \gamma^{d_\text{max}} \left( \frac{1}{s} \sum_{j=1}^s  |C_j|^2 \right)  \\ 
& \underset{\mbox{Jensen}}{\ge}  \frac{s}{n} \gamma^{d_\text{max}}  \left( \frac{1}{s}\sum_{j=1}^s |C_j| \right) ^2 \ge  \gamma^{d_\text{max}} \frac{\theta \alpha^2 n}{4x}. 
\end{align*}

By setting $\theta = \frac{x}{n^{2/3}}$ we guarantee that the value of RA is at least $\gamma^{d_\text{max}}\alpha^2 n^{1/3} / 4$. Since $d_\text{max}-d_\text{min}\leq x$,
$$
\frac{\text{RA}}{\text{OPT}} \ge \frac{ \gamma^{d_\text{max}} \alpha^2 n^{1/3} / 4}{\gamma^{d_\text{min}} \alpha n} \ge \frac{\alpha \gamma ^{x}}{4n^{2/3}} = \frac{\alpha}{2n^{2/3}} ,
$$ where the last inequality follows from the triangle inequality.

 \textbf{Step 2.} 
Assume that OPT gets its value from $n'<n$ rewards that it collects in a segment of length $x$ (and from all other rewards OPT collects a negligible value). Recall that the NN-RA policy is either NN with probability $0.5$ or RA with probability $0.5$. By picking the single reward closest to the starting point, NN gets at least $1/n'$ of the value of OPT. Otherwise, with probability $n'/n$, RA starts with one of the $n'$ rewards picked by OPT and then, by the analysis of step 1, if it sets $\theta = \frac{x}{(n')^{2/3}}$, RA collects  $\frac{1}{2(n')^{2/3}} $  of the value collected by OPT (we use Step (1) with $\alpha =1$). It follows that $$
\frac{\text{NN-RA}}{\text{OPT}} \ge \frac{1}{2} \cdot \frac{1}{n'} + \frac{1}{2} \cdot \frac{n'}{n} \cdot \frac{1}{2(n')^{2/3}}  = \frac{1}{2n'} + \frac{(n')^{1/3}}{4n}.
$$

This lower bound is smallest when $n' \approx n^{\frac{3}{4}}$, in which case NN-RA collect $\Omega(n^{-3/4})$ of OPT.

\textbf{Step 3.}
By the same arguments from Step 3 in the analysis of NN-RDFS, it follows that $$
\frac{\text{NN-RA}}{\text{OPT}} \ge 
\begin{cases} 
\Omega \left( \frac{n^{-\frac{2}{3}}}{ \log(n)} \right), &\text{if} \enspace \text{OPT} = \Omega (n). \\
\Omega \left( \frac{n^{-\frac{3}{4}}}{\log(n)} \right), &\text{otherwise}. \\
\end{cases} 
$$
\end{proof}
\section{R-NN}
We now analyze the performance guarantees of the R-NN method. The analysis is conducted in two steps. In the first step, we assume that OPT achieved a value of $\Omega (n')$ by collecting $n'$ rewards and consider the case that $n' = \alpha n$. The second step considers the more general case and analyzes the performance of NN-Random for the worst value of $n'.$ We emphasize that unlike the previous two algorithms, we do not assume this time that OPT collects its rewards at a segment of length $x$\footnote{Therefore, we do not perform a third step like we did in the analysis of the previous methods.}. 

\begin{theorem*}[R-NN]
For a discounted reward traveling salesman graph with n nodes: $$
\frac{\text{R-NN}}{\text{OPT}} \ge \Omega \left( \frac{\log (n)}{n} \right)
$$

\end{theorem*}

\begin{proof}
\textbf{Step 1.}
Assume OPT collects $\Omega \left(n\right)$ rewards. Define $x = \mbox{log}_\frac{1}{\gamma} (2)$ and $\theta=\frac{x}{\sqrt{n}}$ (here we can replace the $\sqrt{n}$ by any fractional power of $n$, this will not affect the asymptotics of the result) and denote by $\{C_j\}$ the CCs that are obtained by pruning edges longer than $\theta$. We define a CC to be large if it contains more than $\log (n)$ rewards. Observe that since there are at most $\sqrt{n} \enspace$ CCs (Lemma \ref{lemma:theta}), at least one large CC exists. 

\begin{lemma}
\label{lemma:budget1}
Assume that $s_1$ is in a large component $C$. Let $p$ be the path covered by NN starting from $s_1$ until it reaches $s_i$ in a large component.
Let $d$ be the length of $p$ and let $r_1$ be the number of rewards collected by NN in $p$ (including the last reward in $p$ which is back in a large component, but not including $s_1$). Note that $r_1\ge 1$. Then
 $d \le (2^{r_1}-1)\theta$.
\end{lemma}
\begin{proof}
Let $p_i$ be the prefix of $p$ that ends at the $i$th reward on $p$ ($i\le r_1$) and let $d_i$ be the length of $p_i$.
Let $\ell_i$ be the distance from the $i$th reward on $p$ to the $(i+1)$th reward on $p$. 
Since when NN is at the $i$th reward on $p$, the neighbor of $s_1$ in $C$ is at distance at most
$d_i + \theta$ from this reward we have that $\ell_i \le d_i + \theta$. Thus, 
$d_{i+1}\le 2d_i + \theta$ (with the initial condition $d_0 = 0$).
The solution to this recurrence is $d_i = (2^i-1)\theta$.
\end{proof}

\begin{lemma}
\label{lemma:budget2}
For $k<\log (n)$, we have that
after $k$ visits of R-NN in large CCs, for any $s$ in a large CC there exists an unvisited reward at  distance shorter than $(k+1)\theta$ from $s$.
\end{lemma}

\begin{proof}
Let $s$ be a reward in a large component $C$. We have collected at most $k$ rewards from $C$.
Therefore, there exists a reward $s' \in C$ which we have not collected at distance at most $(k+1) \theta$ from $s$. 
\end{proof}

Lemma \ref{lemma:budget1} and Lemma \ref{lemma:budget2} imply the following corollary.

\begin{corollary}
\label{lemma:budget3}
Assume that  $k<\log (n)$, and let
 $p$ be the path of NN from its $k$th reward in a large CC to its $(k+1)$st reward in a large connected component.
Let $d$ denote the length of $p$ and 
$r_k$ be the number of rewards on $p$ (excluding the first and including the last). Then 
 $d\le (2^{r_k}-1)(k+1)\theta \le 2^{r_k+1}k\theta $. 
\end{corollary}

The following lemma concludes the analysis of this step.

\begin{lemma}
\label{lemma:sumri2ri}
Let $p$ be the prefix of R-NN of length $x$.
Let $k$ be the number of segments on $p$ of R-NN  that connect rewards in large CCs and contain internally rewards in small CCs.
For $1\le i\le k$, let $r_i$ be the number of rewards R-NN collects in the $i$th segment. Then $\sum_{i=1}^k r_i = \Omega(\log n)$.
(We assume that $p$ splits exactly into $k$ segments, but in fact the last segment may be incomplete, this requires a minor adjustment in the proof.)
\end{lemma}
\begin{proof}
since $\forall i, r_i\ge 1,$ then if $k \ge \log (n)$ the lemma follows.
So assume that $s < \log n$.
By Corollary \ref{lemma:budget3} we have that 
\begin{equation} \label{eq:hh}
x \le \sum_{i=1}^k 2^{r_i+1}i\theta \le 2^{r_{max}+2}k^2\theta \ .
\end{equation}
where $r_{max} = \mbox{argmax} \enspace \{r_i\}_{i=1}^{k}$.
Since $\theta = x/\sqrt{n}$, Equation (\ref{eq:hh}) implies that 
$\sqrt{n} \le 2^{r_{max}+2}k^2$ and since $k\le \log n$ we get that 
$\sqrt{n} \le 2^{r_{max}+2}\log^2(n)$. Taking logs the lemma follows. 
\end{proof}

Lemma \ref{lemma:sumri2ri} guarantees that once at $s_1 \in C_j,$ R-NN collects  $\Omega(\log (n))$ rewards before traversing a distance of $x$.  Next, notice that the chance that $s_1$ (as defined in Algorithm \ref{alg:nn-bfs}) belongs to one of the large CCs is $p = \frac{n-\sqrt{n}\log (n)}{n}$, 
which is larger than $1/2$ for $n \ge 256$. 

Finally, similar to NN-RDFS, assume that the value of OPT is greater than a constant fraction of $n$, i.e., $OPT \ge n / 2^\alpha.$ This means that OPT must have collected the first $n / 2^{\alpha+1}$ rewards after traversing a distance of at most $\tilde{x} = \left( \alpha + 1\right)x,$ \footnote{To see this, recall that after traversing a distance of $\tilde{x}$, OPT achieved less than $n / 2^{\alpha+1}$. Since it already traversed $\tilde{x}$ it can only achieve less than $n / 2^{\alpha+1}$ from the remaining rewards, thus a contraction with the assumption that it achieved more than $n / 2^\alpha$.}, and denote this fraction of the rewards by $S_\text{OPT}.$ Further denote by $d_{\text{min}},d_{\text{max}}$ the shortest and longest distances from $s_0$ to $S_\text{OPT}$ respectively. By the triangle inequality, $d_{\text{max}} - d_{\text{min}} \le \tilde{x};$ therefore, with a constant probability of $\frac{1}{2^{\alpha+1}},$ we get that $s_1 \in S_\text{OPT}.$ By taking expectation over the first random pick, it follows that

$$
\frac{\text{R-NN}}{\text{OPT}} \ge \frac{1}{2^{\alpha +1}} \frac{\gamma ^{d_\text{max}} \log (n)}{\gamma^{d_\text{min}} n} =\frac{\log (n)}{4^{\alpha +1}n} =\Omega \left(\frac{\log (n)}{n}\right).
$$

\textbf{Step 2.}

Similar to the analysis of NN-RDFS, we now assume that OPT collects its value from $n'<n$ rewards that it collects in a segment of length $x$ (and from all other rewards OPT collects a negligible value). Recall that the R-NN is either NN with probability $0.5$ or a random pick with probability $0.5$ followed by NN. By picking the single reward closets to the starting point, NN gets at least $1/n'$ of the value of OPT. Notice, that we do not need to assume anything about the length of the tour that OPT takes to collect the $n'$ rewards (since we didn't use it in Step 1). It follows that: $$
\frac{\text{R-NN}}{\text{OPT}} \ge \frac{1}{2} \cdot \frac{1}{n'} + \frac{1}{2} \cdot \frac{n'}{n} \cdot \frac{\log (n')}{n'} = \frac{1}{2n'} + \frac{\log (n')}{n} 
$$
Thus, in the worst case scenario, $n'\log (n') \approx n$,
which implies that 
$n' = \Theta(\frac{n}{\log (n)})$. Therefore 
 $\frac{\text{R-NN}}{\text{OPT}} = \Omega \left( \frac{\log (n)}{n} \right)$. 

\end{proof}
\newpage

\section{Exact solutions for the RD-TSP}
\label{sec:exact}
We now present a variation of the Held-Karp algorithm for the RD-TSP. Note that similar to the TSP, $C(\{S,k\})$ denotes the length of the tour visiting all the cities in $S$, with $k$ being the last one (for TSP this is the length of the shortest tour). However, our formulation required the definition of an additional recursive quantity, $V(\{S,k\})$, that accounts for the value function (the discounted sum of rewards) of the shortest path. Using this notation, we observe that Held-Karp is identical to doing tabular Q-learning on SMDP $M_s$. Since Held-Karp is known to have exponential complexity, it follows that solving $M_s$ using SMDP algorithms is also of exponential complexity. 

\begin{algorithm}[h]
\caption{The Held-Karp for the TSP (blue) and RD-TSP (black)}
\label{alg:held-karp}
\begin{algorithmic}
\STATE {\bfseries Input:} Graph $G$, with $n$ nodes

\FOR{$k$ := $2$ to $n$}
\STATE    $\color{blue}{C(\{k\}, k) := d_{1,k}}$
\STATE    $C(\{k\}, k)$ := $d_{1,k}$
\STATE    $V(\{k\}, k)$ := $\gamma^{d_{1,k}} $
\ENDFOR
\FOR{$s$ := $2$ to $n-1$}
    \FOR{all $S \subseteq \{2, . . . , n\}, |S| = s$}
         \FOR{all $k \in S$}
             \STATE $\color{blue}{C(S, k) =\mbox{min}_{m \ne k ,m \in s } [C(S         \setminus \{k\}, m) + d_{m,k}]}$
            \STATE $Q(S,k,a) = [V(S \setminus \{k\}, a) + \gamma ^{C(S \setminus \{k\},a)} \cdot \gamma ^{d_{a,k}} ]$
            \STATE $a^*$ = $\mbox{arg max}_{a \ne k ,a \in s } \enspace Q(S,k,a)$
            \STATE $C(S, k) = C(S \setminus \{k\},a^*) + d_{a^*,k}$
            \STATE $V(S, k) = Q(S,k,a^*)$
        \ENDFOR
    \ENDFOR
\ENDFOR
\STATE $\color{blue}{opt := \mbox{min}_{k\ne 1} [C(\{2,  . . . , n\}, k) + d_{k,1}]}$\\
\STATE $opt := \mbox{max}_{k\ne 1} [V(\{2, \ldots , n\},k) + \gamma ^{C(\{2,  \ldots , n\}, k)+d_{k,1}} $\\
return ($opt$)\\
\end{algorithmic}
\end{algorithm}

\subsection{Exact solutions for simple geometries }
We now provide exact, polynomial time solutions based on dynamic programming for simple geometries, like a  line and a star. We note that such solutions (exact and polynomial) cannot be derived for general geometries.
%as we will see in the proofs of the following Sections, even for a 2d grid. 

\subsubsection{Dynamic programming on a line (1D)}
Given an RD-TSP instance, such that all the rewards are located on a single line (denoted by the integers $1,\ldots,n$ from left to right), it is easy to see that an optimal policy collects, at any time, either the nearest reward to the right or the left of its current location. Thus, at any time the set of rewards that it has already collected lie in a continuous interval between the first uncollected reward to the left of the origin, denoted $\ell$,  and the first uncollected reward to the right of the origin denoted $r$.  The action to take next is either to collect $\ell$ or to collect $r$.

It follows that the state of the optimal agent is uniquely defined by a triple $(\ell,r,c)$, where $c$ is the current location of the agent. Observe that $c \in \{\ell+1,r-1\} $  and therefore there are $O(n^2)$ possible states in which the optimal policy could be. 

Since we were able to classify a state space of polynomial size which contains all states of the optimal policy, then we can describe a 
dynamic programming scheme (Algorithm \ref{alg:line}) that finds the optimal policy. The algorithm
computes a table $V$, where
$V(\ell,r,\rightarrow)$ is
the maximum value we can get by collecting all rewards $1,\ldots \ell$ and $r,\ldots, n$
starting from $r-1$, and  
$V(\ell,r,\leftarrow)$ is defined analogously starting from $\ell+1$. The algorithm
 first initializes the entries of  $V$ 
 where either $\ell = 0$ or $r=n+1$. These entries correspond to the cases where all the rewards to the left (right) of the agent have been collected.
(In these cases the agent continues to collect all remaining rewards one by one in their order.) It then iterates over $t$, a counter over the number of rewards that are left to collect. For each value of $t$, we define $S$ as all the combinations of partitioning these $t$ rewards to the right and the left of the agent.
We fill $V$ by increasing value of $t$. 
To fill an entry $V(\ell,r,\leftarrow)$  such that $\ell + (n+1 -r) = t$ we take the largest among 1) the value to collect $\ell$ and then
the rewards $1,\ldots \ell-1$ and $r,\ldots, n$, appropriately discounted and 2) the value to collect $r$ and then 
$1,\ldots \ell$ and $r+1,\ldots, n$. 
We fill $V(\ell,r,\rightarrow)$ analogously.

The optimal value for starting position $j$ is $1+V(j-1,j+1,\rightarrow)$.
Note that the Algorithm computes the value function; to get the policy, one has merely to track the argmax at each maximization step.

\begin{algorithm}[h]
\caption{Exact solution on a line}
\label{alg:line}
\begin{algorithmic}
\STATE {\bfseries Input:} Graph G, with n nodes
\STATE Init $V(\cdot,\cdot,\cdot)=0$
\FOR{$t=1,n$}
    \STATE $V(\ell=t,n+1,\rightarrow) := \gamma^{d_{t,n}} \cdot \left(1+\sum _{j=t} ^{2}         \gamma^{\sum _{i=0}^{j} d_{i,i-1}}\right)  $ 
    \STATE $V(0,r=t,\leftarrow) := \gamma^{d_{1,t}} \cdot \left(1+\sum _{j=t} ^{n-1}             \gamma^{\sum _{i=0}^{j} d_{i,i+1}}\right)  $ 
\ENDFOR
\FOR{$t = 2, .. n-1  $}
    \STATE $S = \{ (i,n + 1 - j) | i+j = t \}$
    \FOR{$(\ell,r) \in S$}
        \IF{$V(\ell,r,\leftarrow) = 0$}\vspace{-0.4cm}          
        \STATE \begin{equation*}\vspace{-0.4cm}
        V(\ell,r,\leftarrow) = \mbox{max} \begin{cases}
                \gamma ^{d_{\ell,\ell+1}}  \left[ 1 + V(\ell-1,r,\leftarrow) \right] \\
                \gamma^{d_{\ell+1,r}} \left[ 1 + V(\ell,r+1,\rightarrow) \right] 
             \end{cases}
        \end{equation*}
        \ENDIF                   
        \IF{$V(\ell,r,\rightarrow) = 0$}\vspace{-0.4cm}  
        \STATE \begin{equation*} \vspace{-0.4cm}
            V(\ell,r,\rightarrow) = \mbox{max} \begin{cases}
                \gamma ^{d_{\ell,r-1}}  \left[ 1 + V(\ell-1,r,\leftarrow) \right] \\
                \gamma^{d_{r-1,r}} \left[ 1 + V(\ell,r+1,\rightarrow) \right] 
             \end{cases}
        \end{equation*}
        \ENDIF                     
    \ENDFOR
\ENDFOR
\end{algorithmic}
\caption{Optimal solution for the RD-TSP on a line.
The rewards are denoted by $1,\ldots, n$ from left to right. We denote by $d_{i,j}$  the distance between reward $i$ and  reward $j$. 
We denote by $V(\ell,r,\rightarrow)$
the maximum value we can get by collecting all rewards $1,\ldots, \ell$ and $r,\ldots,n$ starting from reward $r-1$. 
 Similarly, we  denote by
$V(\ell,r,\leftarrow)$ maximum value we can get by collecting all rewards $1,\ldots, \ell$ and $r,\ldots,n$
starting from $\ell +1$.
If the leftmost (rightmost) reward was collected we define $\ell=0$ ($r=n+1$).}
\end{algorithm}

\subsubsection{Dynamic programming on a $d$-star}
We consider an RD-TSP instance, such that all the rewards are located on a d-star, i.e., all the rewards are connected to a central connection point via one of $d$ lines and there are $n_i$ rewards along the i$th$ line. We denote the rewards on the $i$th line by $m^i_j \in \{1,..,n_i\},$ ordered from the origin to the end of the line, and focus on the case where the agent starts at the origin.\footnote{The more general case is solved by applying Algorithm \ref{alg:line} until the origin is reached followed by Algorithm \ref{alg:dstar}} It is easy to see that an optimal policy collects, at any time, the uncollected reward that is nearest to the origin along one of the $d$ lines. Thus, at any time the set of rewards that it has already collected lie in $d$ continuous intervals between the origin and the first uncollected reward along each line, denoted by $\bar{\ell}=\{\ell_i\}_{i=1}^d$. The action to take next is to collect one of these nearest uncollected rewards. It follows that the state of the optimal agent is uniquely defined by a tuple $(\bar{\ell},c)$, where $c$ is the current location of the agent. Observe that $c \in \{m^i_{\ell_i-1} \}_{i=1}^d$  and therefore there are $O(dn^{d})$ possible states in which the optimal policy could be. 

Since we were able to classify a state space of  polynomial size which contains all states of the optimal policy then we can describe a dynamic programming scheme (Algorithm \ref{alg:dstar}) that finds the optimal policy. The algorithm computes a table $V$, where $V(\bar{\ell},c)$ is the maximum value we can get by collecting all rewards $\{m^i_{\ell_i} ,\ldots, m^i_{n_i}\}_{i=1}^d$ starting from $c$. The algorithm first initializes the entries of  $V$  where all $\ell_i = n_i+1$ except for exactly one entry. These entries correspond to the cases where all the rewards have been collected, except in one line segment (in these cases the agent continues to collect all remaining rewards one by one in their order.) It then iterates over $t$, a counter over the number of rewards that  are left to collect. For each value of $t$, we define $S$ as all the combinations of partitioning these $t$ rewards among $d$ lines. We fill $V$ by increasing value of $t$.  To fill an entry $V(\bar{\ell},c)$  such that $\sum l_i = n-t$ we take the largest among the values for collecting $\ell_i$ and then the rewards $m^1_{\ell_1} \ldots m^1_{n_1},\ldots, m^i_{\ell_i+1}, \ldots m^i_{n_i},\ldots,m^d_{\ell_d},\ldots,m^d_{n_d} $ appropriately discounted. 

Note that the Algorithm computes the value function; to get the policy, one has merely to track the argmax at each maximization step.

\begin{algorithm}
\caption{Exact solution on a d-star}
\label{alg:dstar}
\begin{algorithmic}
\STATE {\bfseries Input:} Graph G, with n rewards.
\STATE Init $V(\cdot,\cdot)=0$
\FOR{$i \in \{1,..,d\}$}
\FOR{$\ell_i \in \{1,..,n_i\}$}
\FOR{$c \in \{m^1_{n_1},..,m^i_{\ell_i-1},m^d_{n_d}\}$}
        \STATE $V\left(n_1+1,..,l_i,n_d+1,c\right) =  \gamma^{d_{c,m^i_{\ell_i}}}\cdot\left(1+\sum _{j=\ell_i}                                 \gamma^{\sum _{k=0}^{j}d_{k,k+1}}\right)$ 
\ENDFOR
\ENDFOR
\ENDFOR
\FOR{$t = 2, .. n-1  $}
    \STATE $S = \{ \bar{\ell} | \ell_i \in \{1,..,n_i\} \sum \ell_i = n-t \}$
    \FOR{$\bar{\ell} \in S, c \in \{m^i_{\ell_i-1} \}_{i=1}^d$}
        \IF{$V(\bar{\ell},c) = 0$}  
            \STATE $A = \{ m^i_j | j = \ell_i , j\le n_i\} $ 
            \FOR{$a \in A $}
                    \STATE \begin{equation*}%\vspace{-0.8cm} 
                    V(\bar{\ell},c)= \mbox{max} \begin{cases} 
                    V(\bar{\ell},c)\\
                    \gamma ^{d_{c,a}}  \left[ 1 + V(\bar{\ell}+e_a,a) \right]
                    \end{cases}\end{equation*}
            \ENDFOR
        \ENDIF                   
    \ENDFOR
\ENDFOR
\end{algorithmic}
\caption{Optimal solution for the RD-TSP on a d-star. We denote by $n_i$ the amount of rewards there is to collect on the $i$th line, and denote by $m^i_j \in \{1,..,n_i\}$ the rewards along this line, from the center of the star to the end of that line. We denote by $d_{m^t_i,m^k_j}$  the distance between reward $i$ on line $t$ and reward $j$ on line $k$. The first uncollected reward along each line is denoted by $\ell_i$,
and the maximum value we can get by collecting all the remaining rewards $m^1_{\ell_1} \ldots m^1_{n_1},\ldots,m^d_{\ell_d},\ldots,m^d{n_d}$ starting from reward $c$ is defined by $V(\bar{\ell}=\{\ell_i\}_{i=1}^d,c)$.  If all the rewards were collected on line $i$ we define $\ell_i=n_i+1$.}
\end{algorithm}

\newpage

\end{document}